\documentclass{aamas2013}

\usepackage{graphicx}
\usepackage{epsfig}

\usepackage{amsmath}
\usepackage{amssymb,amsfonts,textcomp}
\usepackage{color}
\usepackage{array}
\usepackage{hhline}

\usepackage{algorithm}	
\usepackage{algorithmic}
\usepackage{url}

\newtheorem{definition}{Definition}[section]
\newtheorem{example}{Example}[section]
\newtheorem{proposition}{Proposition}[section]
\newtheorem{theorem}{Theorem}[section]
\newtheorem{lemma}[theorem]{Lemma}

\newcommand{\nop}[1]{{}}

\def\smallblacksq{{\tiny $\blacksquare$}}

\def\dini{d_{i}^{in}}
\def\dinm{{d_{*}^{in}}}
\def\FB{\mathit{FBnd}}
\def\RB{\mathit{RBnd}}
\def\IB{\mathit{IBnd}}
\def\ELIG{\mathit{Elig}}
\def\QUAL{\mathit{Qual}}

\def\BOUND{\mathit{Bound}}
\def\true{\textsf{Tr}}
\def\false{\textsf{F}}
\def\bnd{\mathit{bnd}}
\def\strtie{\textit{strTie}}
\def\wktie{\textit{wkTie}}
\def\male{\textit{male}}
\def\female{\textit{fem}}
\def\watchesA{\textit{visPgA}}
\def\watchesB{\textit{visPgB}}

\def\gedge{g_{edge}}
\def\gnode{g_{node}}
\def\mancalog{\textsf{MANCaLog}}

\def\tip{\mathit{tip}}
\def\softtip{{\mathit{sftTp}}}

\def\negtip{{\mathit{ngTp}}}

\def\ifl{\mathit{ifl}}

\def\IFL{\mathit{IFL}}

\def\setN{\textbf{N}}
\def\calf{\mathcal{F}}

\def\tarrow{\stackrel{\Delta t}{\leftarrow}}

\def\thrarrow{\stackrel{3}{\leftarrow}}

\def\twoarrow{\stackrel{2}{\leftarrow}}
\def\onearrow{\stackrel{1}{\leftarrow}}

\def\call{\mathcal{L}}
\def\calg{\mathcal{G}}

\def\cali{\mathcal{I}}
\def\inf{\mathit{inf}}
\def\sup{\mathit{sup}}

\def\tm{t_\mathit{max}}

\def\S{\textbf{S}}
\def\<{\langle}
\def\>{\rangle}

\def\call{\mathcal{L}}

\def\calf{\mathcal{F}}

\begin{document}

\title{MANCaLog: A Logic for Multi-Attribute Network Cascades (Technical Report)}

\numberofauthors{3}
\author{
%
%
\alignauthor
Paulo Shakarian\\
       \affaddr{Network Science Center and}\\
       \affaddr{Dept. of Electrical Engineering}\\
       \affaddr{and Computer Science}\\
       \affaddr{U.S. Military Academy}\\
       \affaddr{West Point, NY 10996}\\
       \email{paulo@shakarian.net}
\alignauthor
Gerardo I. Simari\\
       \affaddr{Dept. of Computer Science}\\
       \affaddr{University of Oxford}\\
       \affaddr{Wolfson Building, Parks Road}\\
       \affaddr{Oxford OX1 3QD, UK}\\
       \email{gerardo.simari@cs.ox.ac.uk}
\alignauthor Robert Schroeder\\
       \affaddr{CORE Lab}\\
       \affaddr{Defense Analysis Dept.}\\
       \affaddr{Naval Postgraduate School}\\
       \affaddr{Monterey, CA 93943}\\
       \email{rcschroe@nps.edu}
}

\maketitle

\begin{abstract}
The modeling of cascade processes in multi-agent systems in the form of complex networks has in recent years become 
an important topic of study due to its many applications: the adoption of commercial
products, spread of disease, the diffusion of an idea, etc. In this paper, we begin by identifying a desiderata of seven properties that a framework for modeling such processes should satisfy: 
the ability to represent attributes of both nodes and edges, 
an explicit representation of time,
the ability to represent non-Markovian temporal relationships, 
representation of uncertain information, 
the ability to represent competing cascades,
allowance of non-monotonic diffusion, and 
computational tractability. We then present the $\mancalog$ language, a formalism
based on logic programming that satisfies all these desiderata, and focus on algorithms for finding minimal 
models (from which the outcome of cascades can be obtained) as well as how this formalism can be applied in real world scenarios. We are not aware of any other formalism in the literature that
meets all of the above requirements.
\end{abstract}






\section{Introduction and Related Work}
\label{sec:intro}

An epidemic working through a population, cascading electrical power failures, product adoption, and the spread of a mutant gene are all examples of diffusion processes that can happen in multi-agent systems structured as complex networks.  These network processes have been studied in a variety of disciplines, including computer science~\cite{kleinberg}, biology~\cite{lieberman_evolutionary_2005}, sociology~\cite{Gran78}, economics~\cite{schelling}, and physics~\cite{sood08}.  Much existing work in this area is based on pre-existing models in sociology and economics -- in particular the work of~\cite{Gran78,schelling}.  However, recent examinations of social networks -- both analysis of large data sets and experimental -- have indicated that there may be additional factors to consider that are not taken into account by these models.  These include the attributes of nodes and edges, competing diffusion processes, and time.  In this paper, we outline seven design criteria (Section~\ref{sec:criteria}) for such a framework and introduce $\mancalog$ (Section~\ref{prelimSec}), which is to the best of our knowledge the first logical language for modeling diffusion in  complex networks that meets these criteria.  $\mancalog$ is a rule-based framework (inspired by logic programming) that can richly express how agents adopt or fail to adopt certain behaviors, and how these behaviors cascade through a network.  We also introduce fixed-point based algorithms that allow for the calculation of the result of the diffusion process in Section~\ref{fpSec}.  Note that these algorithms are proven not only to be correct, but also to run in polynomial time.  Hence, our approach can not only better express many aspects of cascades in complex networks, but it can do so in a reasonable amount of time.  We conclude by discussing applications of $\mancalog$ in Section~\ref{sec:learn}.

\medskip
\noindent
{\em Proofs of all results stated in this paper can be found in the appendix.}

\subsection{Desiderata of Properties}
\label{sec:criteria}

We begin by identifying a set of criteria that we believe a framework for reasoning about cascades in complex networks should satisfy.

\smallskip
\noindent
\textbf{1.\ Multiply labeled and weighted nodes and edges.}
Many existing frameworks for studying diffusion in complex networks assume that there is only one type of vertex that may become ``active''~\cite{kleinberg} or may ``mutate''~\cite{lieberman_evolutionary_2005,sood08} and only one possible relationship between nodes.  In reality, nodes and edges often have different properties.  For instance, labels on edges can be used to differentiate between strong and weak ties (edge types) -- a concept that is well studied~\cite{granovetter1973}.  Recently, such attributes of nodes have been shown to impact influence in a network~\cite{aral12}.

\smallskip
\noindent
\textbf{2.\ Explicit Representation of Time.}
Most work in the literature assumes static models, with the exception of the recent developments
in~\cite{goyal2008discovering,goyal2010learning,goyal2011data}, which assume the existence of a timestamped log
referring to actions taken in the network in order to learn how nodes influence each other.
Though~\cite{goyal2008discovering} tackles the problem of predicting the time at which a certain node will take an
action, the authors make several simplifying assumptions such as monotonicity of probability functions, probabilistic
independence, sub-modularity and, most importantly for this criterion, a modeling of time solely based on temporal decay
of influence. We seek a richer model of temporal relationships between conditions in the network structure, the
current state of the cascades in process, and how influence propagates.

\smallskip
\noindent
\textbf{3.\ Non-Markovian Temporal Relationships.}
Apart from time being explicitly represented, the temporal dependencies should be able to span multiple units of time.  Hence, the ``memoryless'' mode of a standard Markov process, where only the information of the current state is required, is insufficient.  Here, we strive to create a framework where dependencies can be from other earlier time steps.  This issue has been previously studied with respect to more general logic programming frameworks such as~\cite{apt1}, but to our knowledge has not been applied to social networks.

\smallskip
\noindent
\textbf{4.\ Representation of Uncertainty.}
As in practice it is not always possible to judge the attributes of all individuals in a network, an element of uncertainty must be included.  However, in connection with point~7, this should not be at the expense of tractability.  For instance, the probabilistic models of~\cite{kleinberg} are normally addressed with simulation (and hence do not scale well) as the computation of the expected number of activated nodes is a $\#P$-hard problem~\cite{chen10}.

\smallskip
\noindent
\textbf{5.\ Competing Cascades.}
Often, in real-world situations there will be competing cascading processes.  For example, in evolutionary graph theory~\cite{lieberman_evolutionary_2005}, ``mutants'' and ``residents'' compete for nodes in the network -- the success of one hinges on the failure of the other.

\smallskip
\noindent
\textbf{6.\ Non-Monotonic Cascades.}
In much existing work on cascades in complex networks, the number of nodes attaining a certain property at each time step can only increase.  However, if we allow for competing cascades in the same model, we cannot have such a strong restriction as the success of one cascade may come at the expense of another.

\smallskip
\noindent
\textbf{7.\ Tractability.}
The social networks of interest in today's data mining problems often have millions of nodes.  It is reasonable to expect that soon billion-node networks will be commonplace.  Any framework for dealing with these problems must be solvable in a reasonable amount of time and offer areas for practical improvement for further scalability.

\subsection{Related Work}
\label{sec:rel-work}

The above criteria can be summarized as the desire to design the most expressive language for network cascades possible while still allowing computation of the outcome of a diffusion process to be completed in a tractable amount of time.  As a comparison, let us briefly describe some relevant related work.  Perhaps the best known general model for representing diffusion in complex networks is the independent cascade/linear threshold (IC/LT) model of~\cite{kleinberg}. However, although this framework was shown to be capable of expressing a wide variety of sociological models, it assumes the Markov property and does not allow for the representation of multiple attributes on vertices and edges. A more recent framework, social network optimization problems (SNOPs)~\cite{snops} uses logic programming to allow for the representation of attributes, but this framework does not allow for competing processes or non-monotonic cascades. A related logic programming framework, competitive diffusion (CD)~\cite{bss10} allows for competitive diffusion and non-monotonic processes but does not explicitly represent time and also makes Markovian assumptions. Further, we also note that the semantics of CD yields a ``most probable interpretation'' that is not a unique solution. Hence, a given model in that framework can lead to multiple and possibly contradictory, outcomes to a cascade (this problem is avoided in $\mancalog$).  Another popular class of models is Evolutionary Graph Theory (EGT)~\cite{lieberman_evolutionary_2005}, which is highly related to the voter model (VM)~\cite{sood08}. Although this framework allows for competing processes and non-monotonic diffusion, it also makes Markovian assumptions while not explicitly representing time.  Further, determining the outcome of a cascade in those models is NP-hard, while determining the outcome in $\mancalog$ can be accomplished in polynomial time.  Table~\ref{rwTab} lists how these models compare to $\mancalog$ when considering our design criteria.

\begin{table*}
\caption{Comparison with other models}
\label{rwTab}
\begin{center}
\begin{tabular}{|l|c|c|c|c|c|c|}
\hline
Criterion  & MANCaLog &  IC/LT \cite{kleinberg} & SNOP \cite{snops} & CD \cite{bss10} & EGT/VM \cite{lieberman_evolutionary_2005}\\
\hline
\hline
1.\ Labels & Yes &No & Yes &Yes &No  \\
\hline
2.\ Explicit Representation of Time & Yes &  No & Yes & No & Yes  \\
\hline
3.\ Non-Markovian Temporal Relationships & Yes &No & No &No & No\\
\hline
4.\ Uncertainty & Yes &   Yes & Yes &Yes & Yes  \\
\hline
5.\ Competing Cascades & Yes & No & No &Yes &Yes  \\
\hline
6.\ Non-monotonic Cascades & Yes &No & No &Yes& Yes \\
\hline
7.\ Tractablity & PTIME & $\#$P-hard & PTIME & PTIME & NP-hard  \\
\hline
\end{tabular}
\end{center}
\end{table*}

\section{Framework}
\label{prelimSec}

\subsection{Syntax and Semantics}

In this paper we assume that agents are arranged in a directed graph (or network) $G=(V,E)$, where the set of nodes corresponds to the agents, and the edges model the relationships between them.
We also assume a set of labels $\call$, which is partitioned into two sets: {\em fluent} labels $\call_{f}$ (labels that can change over time) and {\em non-fluent} labels $\call_{nf}$ (labels that do not); labels can be applied to both the nodes and edges of the network.  We will use the notation $\calg = V \cup E$ to be the set of all \textit{components} (nodes and edges) in the network.  Thus, $c \in \calg$ could be either a node or an edge.

\begin{example}
\label{ex1}
We will use the sample online social network $G_{soc}$ shown in Figure~\ref{nwFig} as the running example; $\calg_{soc}$ is used to denote the set of components of $G_{soc}$.  Here we have $\call_{nf}=\{\male,\female,$ $\strtie,\wktie\}$ representing male, female, strong ties and weak ties, respectively.  Additionally, we have $\call_{f} = \{\watchesA,\watchesB \}$ representing visiting webpage A and visiting webpage B, respectively.
\hfill\smallblacksq
\end{example}

\begin{figure}[t!]
  \centering
  \includegraphics[width=0.45\textwidth]{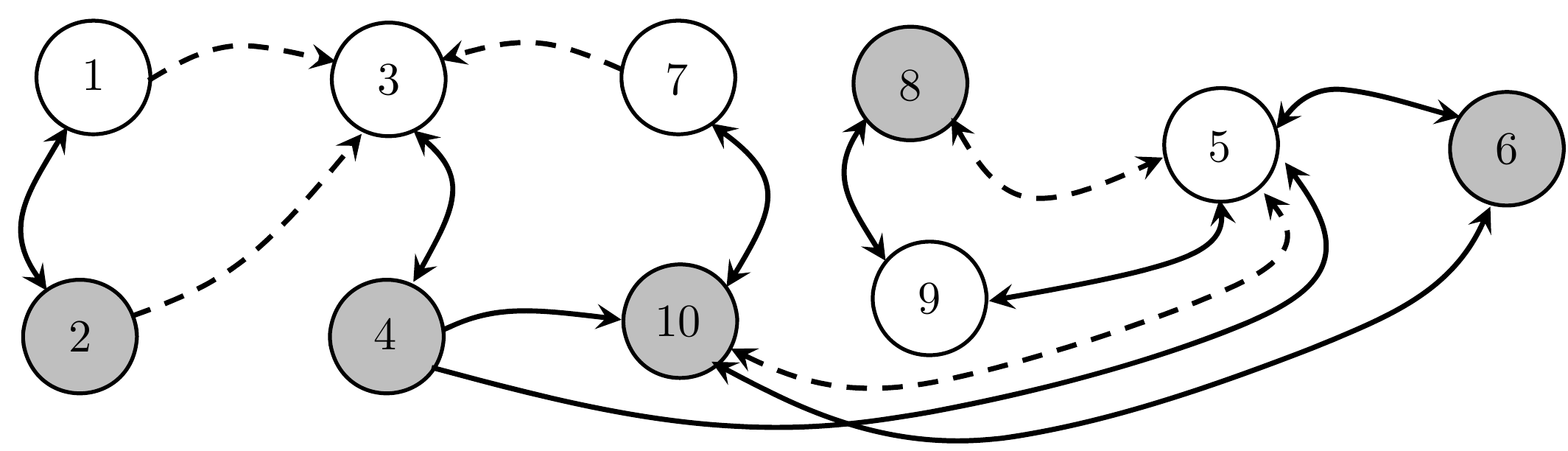}
      \caption{Simple online social network $G_{soc}$.  Solid edges are labeled with $\strtie$, while dashed edges are labeled with $\wktie$.  White nodes are labeled with $\male$, while gray nodes are labeled with $\female$. Arrows represent the direction of the edge; double-headed edges represent two edges with the same label.}\label{nwFig}
\end{figure}

In this paper, we present a logical language where we use atoms, referring to labels and weights, to describe properties of the nodes and edges.  Though labels themselves could be modeled as atoms instead of predicates (to model non-ground labelings that allow for greater expressibility), for simplicity of presentation we leave this to future work.  The first piece of the syntax is the \textbf{network atom}.

\begin{definition}[Network Atom]
Given label $L \in \call$ and weight interval $\bnd \subseteq [0,1]$, then $\langle L, \bnd \rangle$ is a \textbf{network atom}.  A network atom is fluent (resp., non-fluent) if $L \in \call_{f}$ (resp., $L \in \call_{\mathit{nf}}$). We use $NA$ to denote the set of all possible network atoms.
\end{definition}

Network atoms describe  properties of nodes and edges.  The definition is intuitive: $L$ represents a property of the vertex or edge, and associated with this property is some weight that may have associated uncertainty -- hence represented as an interval $\bnd$, which can be open or closed. An invalid bound is represented by $\emptyset$, which is equivalent to all other invalid bounds.

\begin{definition}[World]
A world $W$ is a set of network atoms such that for each $L \in \call$ there is no more than one network atom of the form $\<L,\bnd \>$ in $W$.
\end{definition}

A \textbf{network formula} over $NA$ is defined using conjunction, disjunction, and negation in the usual way.  If a formula contains only non-fluent (resp., fluent) atoms, it is a non-fluent (resp., fluent) formula.

\begin{definition}[Satisfaction of Worlds]
Given a world $W$ and network formula $f$, \textbf{satisfaction} of $W$ by $f$ (denoted $W \models f$) is defined:
\begin{itemize}
\item If $f = \langle L, [0,1]  \rangle$ then $W \models f$.

\item If $f = \langle L, \emptyset  \rangle$ then $W \not\models f$.

\item\label{chgOne}
If $f = \langle L, \bnd \rangle$, with $\bnd \neq \emptyset$ and $\bnd \neq [0,1]$, then 
$W \models f$ iff there exists $\langle L, \bnd' \rangle \in W$ s.t. $\bnd' \subseteq \bnd$.

\item If $f = \neg f'$ then $W \models f$ iff $W \not\models f'$.

\item If $f = f_1 \wedge f_2$ then $W \models f$ iff $W \models f_1$ and $W \models f_2$.

\item If $f = f_1 \vee f_2$ then $W \models f$ iff $W \models f_1$ or $W \models f_2$.
\end{itemize}
\end{definition}

For some arbitrary label $L \in \call$, we will use the notation $\true = \<L,[0,1]\>$ and $\false = \<L,\emptyset \>$ to represent a tautology and contradiction, respectively.  For ease of notation (and without loss of generality), we say that if there does not exist some $\bnd$ s.t.\ $\<L, \bnd\> \in W$, then this implies that $\<L, [0,1]\> \in W$.

\begin{example}
\label{ex2}
Following from Example~\ref{ex1}, the network atom $\< female, [1,1] \>$ can be used to identify a node as a woman.  Likewise, the world $W =$
\[
\Big\{\<\female, [1,1]\>, \<\male, [0,0]\>, \<\watchesA,[1,1]\>, \<\watchesB,[0,0]\>\Big\}
\]
might be used to identify a woman who visits webpage A.  Clearly, we have that $W \models$
\[
\< \female, [1,1] \> \wedge \neg \<\watchesA,[0.5,0.9]\> \wedge \neg \<\watchesB,[0.1,0.7]\>
\]
Note that the network atoms formed with $\strtie$ and $\wktie$ are not present; this could be due to the fact that such a world is used to describe a node and not an edge, and hence there is no information about those two labels.  As such is the case, $W \models \<\strtie,[0,1]\> \wedge \<\wktie,[0,1]\>$.
\hfill\smallblacksq
\end{example}

The idea is to use $\mancalog$ to describe how properties (specified by labels) of the nodes in the network change over time.  We assume that there is some natural number $\tm$ that specifies the total amount of time we are considering, and we use $\tau = \{t \; | \; t \in [0,\tm]\}$ to denote the set of all time points.  How well a certain property can be attributed to a node is based on a \textit{weight} (to which the $\bnd$ bound in the network atom refers).  As time progresses, a weight can either increase/decrease and/or become more/less certain. We now introduce the $\mancalog$ fact, which states that some network atom is true for a node or edge during certain times.

\begin{definition}[$\mancalog$ Fact]
If $[t_1,t_2] \subseteq [0,\tm]$, $c \in \calg$, and $a \in NA$, then $(a,c):[t_1,t_2]$ is a $\mancalog$ \textbf{fact}.  A fact is fluent (resp., non-fluent) if atom $a$ is fluent (resp., non-fluent).  All non-fluent facts must be of the form $(a,c):[0,\tm]$.  Let $\calf$ be the set of all facts and $\calf_{nf}, \calf_{f}$ be the set of all non-fluent and fluent facts,
respectively.
\end{definition}

\begin{example}
\label{ex3}
Following from Example~\ref{ex2}, the following facts are based on Figure~\ref{nwFig}:
\begin{small}
\begin{eqnarray*}
F_1&=&(\<\male,[1,1]\>,1):[0,\tm]\\
F_2&=&(\<\female,[1,1]\>,1):[0,\tm]\\
F_3&=&(\<\male,[1,1]\>,3):[0,\tm]\\
F_4&=&(\<\strtie,[1,1]\>,(1,2)):[0,\tm]\\
F_5&=&(\<\strtie,[1,1]\>,(2,1)):[0,\tm]\\
F_6&=&(\<\wktie,[1,1]\>,(2,3)):[0,\tm]\\
F_7&=&(\<\watchesA,[0.8,1.0]\>,1):[0,\tm]\label{fact7}\\
F_8&=&(\<\watchesA,[0.5,1.0]\>,2):[0,\tm]\label{fact8}
\end{eqnarray*}
\end{small}

\noindent
For instance, agent~1 is male, and has a strong tie to agent~2, who is female.
\hfill\smallblacksq
\end{example}

\noindent
Next, we introduce integrity constraints (ICs).

\begin{definition}
Given fluent network atom $a$ and conjunction of network atoms $b$, an integrity constraint is of
the form $a \hookleftarrow b$.
\end{definition}

Intuitively, integrity constraint $\<L,\bnd\> \hookleftarrow b$ means that if at a certain time point a component (vertex or edge) of the network has a set of properties specified by conjunction $b$, then at that same time the component's weight for label $L$ must be in interval $\bnd$.

\begin{example}
Following from the previous examples, the integrity constraint $\<\male,[0,0]\> \hookleftarrow \<\female,[1,1]\>$ would require any node designated as a female to not be male.
\hfill\smallblacksq
\end{example}

We now define $\mancalog$ rules.  The idea behind rules is simple: an agent that meets some criteria is influenced by the set
of its neighbors who possess certain properties.  The amount of influence exerted on an agent by its neighbors is specified by an \textit{influence function}, whose precise effects will be described later on when we discuss the semantics.
As a result, a rule consists of four major parts: (i) an influence function, (ii) neighbor criteria, (iii) target criteria, and (iv) a target.  Intuitively, (i) specifies how the neighbors influence the agent in question, (ii) specifies which of the neighbors can influence the agent, (iii) specifies the criteria that cause the agent to be influenced, and (iv) is the property of the agent that changes as a result of the influence.

We will discuss each of these parts in turn, and then define rules in terms of these elements.  First, we define influence functions and neighbor criteria.

\begin{definition}[Influence Function]
An \textbf{influence function} is a function $\ifl : \setN \times \setN \rightarrow [0,1] \times [0,1]$ that satisfies
the following two axioms:

\smallskip
\noindent
1.\ $\ifl$ can be computed in constant ($O(1)$) time.

\smallskip
\noindent
2.\ For $x'>x$ we have $\ifl(x',y) \subseteq \ifl(x,y)$.

\smallskip
\noindent
We use $\IFL$ to denote the set of all influence functions.
\end{definition}

Intuitively, an influence function takes the number of qualifying influencers and the number of eligible influencers and returns a bound on the new value for the weight of the property of the target node that changes.
In practice, we expect the time complexity of such a function to be a polynomial in terms of the two arguments. However, as both arguments are naturals bounded by the maximum degree of a node in the network, this value will be much smaller than the size of the network -- we thus treat it as a constant here.

\begin{example}
The well-known ``tipping model'' originally introduced in \cite{Gran78,schelling} states that an agent adopts a behavior when a certain fraction of his incoming neighbors do so.  A common tipping function is the \textbf{majority threshold} where at least half of the agent's neighbors must previously adopt the behavior.  We can represent this using the following influence function:
\begin{equation*}
\tip(x,y) = \begin{cases} [1.0,1.0] & \textit{if } x/y \geq 0.5  \\ [0.0,1.0] & \text{otherwise} \end{cases}
\end{equation*}
This function says that an agent adopts a certain behavior if at least half of his incoming neighbors have some property (strong ties, weak ties, meet some requirement of gender, income, etc.) and that we have no information otherwise.  In our framework, we can leverage the bounds associated with the influence function to create a ``soft'' tipping function:
\begin{equation*}
\softtip(x,y) = \begin{cases} [0.7,1.0] & \textit{if } x/y \geq 0.5  \\ [0.0,1.0] & \text{otherwise} \end{cases}
\end{equation*}
Intuitively, the above function says that an agent adopts a behavior with a weight of at least $0.7$ if half of the incoming neighbors that have some attribute and meet some criteria, and we have no information otherwise.  Another possibility is to have an influence function that may reduce the weight that an agent adopts a certain behavior:
\begin{equation*}
\negtip(x,y) = \begin{cases} [0.0,0.2] & \textit{if } x=y \\ [0.0,1.0] & \text{otherwise} \end{cases}
\end{equation*}
The $\negtip$ function says that an agent will adopt a behavior with a weight no greater than $0.2$ if all of the incoming neighbors possessing some property meet some criteria, and that we have no information otherwise.
\hfill\smallblacksq
\end{example}

\begin{definition}[Neighbor Criterion]
If $\gedge,\gnode$ are non-fluent network formulas (formed over edges and nodes, respectively), $h$ is a conjunction of network atoms, and $\ifl$ is an influence function, then
$
(\gedge,\gnode,h)_\ifl
$
is a neighbor criterion.
\end{definition}
Formulas $\gnode$ and $h$ in a neighbor criterion specify the (non-fluent and fluent, respectively) criteria on a given neighbor, while formula $\gedge$ specifies the non-fluent criteria on the directed edge from that neighbor to the node in question.

The next component is the ``target criteria'', which are the criteria that an agent must satisfy in order to be influenced by its neighbors.  Ideas such as ``susceptibility''~\cite{aral12} can be integrated into our framework via this component.  We represent these criteria with a formula of non-fluent network atoms.  The final component, the ``target'', is simply the label of the target agent that is influenced by its neighbors.  Hence, we now have all the pieces to define a rule.

\begin{definition}[Rule]
\label{ruleDef}
Given fluent label $L$, natural number $\Delta t$, target criteria $f$ and neighbor criteria\\ $(\gedge,\gnode,h)_\ifl$, a $\mancalog$ Rule is of the form:
\begin{eqnarray*}
r = L &\tarrow& f, (\gedge,\gnode,h)_\ifl
\end{eqnarray*}
We will use the notation $head(r)$ to denote $L$.
\end{definition}
Note that the target (also referred to as the head) of the rule is a single label; essentially, the body of the rule characterizes a set of nodes, and this label is the one that is modified for each node in this set.
More specifically, the rule is essentially saying that when certain conditions for an agent and its neighbors are met, the $\bnd$ bound for the network atom formed with label $L$ on that agent changes. Later, in the semantics, we introduce network interpretations, which map components (nodes and edges) of the network to worlds at a given point in time. The rule dictates how this mapping changes in the next time step.

\begin{definition}[$\mancalog$ Program]
A program $P$ is a set of rules, facts, and integrity constraints s.t.\ each non-fluent fact $F \in \calf_{nf}$ appears no more than once in the program.  Let $\mathbf{P}$ be the set of all programs.
\end{definition}

\begin{example}
Following from the previous examples, we can have a $\mancalog$ program that leverage the $\softtip$ and $\negtip$ influence functions in rules that are more expressive than previous models.  Consider the following rules:
\begin{small}
\begin{eqnarray*}
R_1&=&\watchesA \twoarrow \\
&& \<\female,[1,1]\>,(\<\strtie,[0.9,1]\>,\true ,\<\watchesA,[0.9,1.0] \>)_\softtip \label{rule1}\\
R_2&=&\watchesB  \onearrow \\
&&\<\male,[1,1]\>,(\true,\true ,\<\watchesB,[0.8,1.0] \>)_\softtip \label{rule2}\\
R_3&=&\watchesA  \thrarrow \\
&& \<\male,[1,1]\>,(\true,\<\female,[1,1] \>,\neg\<\watchesA,[0.7,1.0] \>)_\negtip \label{rule3}
\end{eqnarray*}
\end{small}
\noindent
Rule $R_1$ says that a female agent in the network visits page A with a weight of at least $0.7$ (this is specified in
the $\mathit{sftTp}$ influence function) if at least half of her strong ties (with weight of at least $0.9$) visited the page (with a weight of at least $0.9$) two days ago.
The rest of the rules can be read analogously.
\hfill\smallblacksq
\end{example}

We now introduce our first semantic structure: the \textbf{network interpretation}.

\begin{definition}[Network Interpretation]
A network interpretation is a mapping of network components to sets of network atoms, $NI : \calg \rightarrow 2^{NA}$.  We will use $\mathbf{NI}$ to denote the set of all network interpretations.
\end{definition}

We note that not all labels will necessarily apply to all nodes and edges in the network.  For instance, certain labels may describe a relationship while others may only describe a property of an individual in the network.  If a given label $L$ does not describe a certain component $c$ of the network, then in a valid network interpretation $NI$, $\< L, [0,1] \> \in NI(c)$.

\begin{example}
\label{niEx}
Consider $G_{soc}'$, the induced subgraph of $G_{soc}$ that has only nodes $\{1,2,3,4,5\}$.  Table~\ref{niExTab} shows the contents of $NI_{1}$, an example network interpretation.
\hfill\smallblacksq
\end{example}

\begin{table}
\scriptsize
\caption{Example network interpretation, $NI_{1}$.}
\label{niExTab}
\begin{center}
\begin{tabular}{|l|l|l|l|l|l|l|}
\hline
Comp.\ & $\male$ & $\female$ & $\strtie$ & $\wktie$ & $\watchesA$ & $\watchesB$\\
\hline
\hline
1 & $[1,1]$& $[0,0]$&-&-& $[0.9,1.0]$ & $[0.8,1.0]$ \\
\hline
2 & $[0,0]$& $[1,1]$&-&-& $[0.0,0.3]$ & $[0.0,0.2]$ \\
\hline
3 & $[1,1]$& $[0,0]$&-&-& $[0.6,1.0]$ & $[0.0,0.2]$ \\
\hline
4 & $[0,0]$& $[1,1]$&-&-& $[0.0,0.2]$ & $[0.9,1.0]$ \\
\hline
5 & $[1,1]$& $[0,0]$&-&-& $[0.0,0.2]$ & $[0.7,1.0]$ \\
\hline
(1,2) & -& - & $[1,1]$ & $[0,0]$ & -& - \\
\hline
(2,1) & -& - & $[1,1]$ & $[0,0]$ & -& - \\
\hline
(1,3) & -& - & $[0,0]$ & $[1,1]$ & -& - \\
\hline
(2,3) & -& - & $[0,0]$ & $[1,1]$ & -& - \\
\hline
(3,4) & -& - & $[1,1]$ & $[0,0]$ & -& - \\
\hline
(4,3) & -& - & $[1,1]$ & $[0,0]$ & -& - \\
\hline
(4,5) & -& - & $[1,1]$ & $[0,0]$ & -& - \\
\hline
\end{tabular}
\end{center}
\end{table}

We define a $\mancalog$ interpretation as follows.

\begin{definition}[Interpretation]
A $\mancalog$ interpretation $I$ is a mapping of natural numbers in the interval $[0,\tm]$ to network interpretations,
i.e., $I : \setN \rightarrow \mathbf{NI}$.  Let $\cali$ be the set of all possible interpretations.
\end{definition}


\subsection{Satisfaction}

First, we define what it means for an interpretation to satisfy a fact and a rule.

\begin{definition}[Fact Satisfaction]
An interpretation $I$ \textbf{satisfies} $\mancalog$ fact $(a,c):[t_1,t_2]$, written $I \models (a,c):[t_1,t_2]$, iff $\forall t \in [t_1,t_2]$, $I(t)(c)\models a$.
\end{definition}

\begin{example}
\label{satFactEx}
Consider interpretation $I_1$, where $I_1(0)=NI_1$ (from Example~\ref{niEx}), and $\mancalog$ facts $F_7$ and $F_8$ from Example~\ref{ex3}.
In this case, $I_1 \models F_7$ and $I_1 \not\models F_8$.
\hfill\smallblacksq
\end{example}

For non-fluent facts, we introduce the notion of strict satisfaction, which enforces the bound in the interpretation to be set to exactly what the fact dictates.

\begin{definition}[Strict Fact Satisfaction]
In\-ter\-pre\-ta\-tion $I$ \textbf{strictly satisfies} $\mancalog$ fact $(c,a):[t_1,t_2]$ iff $\forall t \in [t_1,t_2]$, $a \in I(t)(c)$.
\end{definition}

Next, we define what it means for an interpretation to satisfy an integrity constraint.

\begin{definition}[IC Satisfaction]
An interpretation $I$ \textbf{satisfies} integrity constraint $a \hookleftarrow b$ iff for all $t \in \tau$ and $c \in \calg$, $I(t)(c)\models \neg b \vee a$.
\end{definition}

Before we define what it means for an interpretation to satisfy a rule, we require
two auxiliary definitions that are used to define the bound enforced on a label by a given rule, and
the set of time points that are affected by a rule.

\begin{definition}[$\BOUND$ function]
\label{boundDef}
For a given rule $r = L \tarrow f, (\gedge,\gnode,h)_\ifl$, node $v$, and network interpretation $NI$,
$\BOUND(r,v,NI) = $
\[
\ifl\Big(\Big|\QUAL\big(v,\gedge,\gnode,h,NI\big)\Big|, \Big|\ELIG\big(v,\gedge,\gnode,NI\big)\Big|\Big),
\]
where $\ELIG(v,\gedge,\gnode,NI) =$
\[
\Big\{v' \in V \; | \; NI(v') \models \gnode  \wedge (v',v) \in E\wedge NI\big((v',v)\big)\models \gedge\Big\}
\]
and $\QUAL(v,\gedge,\gnode,h,NI) =$
\[
\Big\{v' \in \ELIG(v,\gedge,\gnode,NI) \;|\;  NI(v') \models h\Big\}
\]
\end{definition}
Intuitively, the bound returned by the function depends on the influence function and the number of qualifying and eligible nodes that influence it.

\begin{definition}[Target Time Set]
For interpretation $I$, node $v$, and rule $r = L \tarrow f, (\gedge,\gnode,h)_\ifl$, the \textit{target time set} of $I,v,r$ is defined as follows:
\[
\mathit{TTS}(I,v,r) = \Big\{ t \in [0,\tm] \; | \;  I(t-\Delta t)(v) \models f\Big\}
\]
We also extend this definition to a program $P$, for a given $c \in \calg$ and $L \in \call$, as follows;
$\mathit{TTS}(I,c,L,P) =$
\begin{small}
\[
\bigcup_{r \in P, \mathit{head}(r)=L}\mathit{TTS}(I,c,r) \cup \big\{t\in [t_1,t_2] \; | \; (\<L,\bnd\>,c):[t_1,t_2] \in P\big\}
\]
\[
\cup \;
\Big\{ t \; | \; \<L, \bnd\> \hookleftarrow b \in P \wedge I(t)(c)\models b\Big\}
\]
\end{small}
\end{definition}

\noindent
We can now define satisfaction of a rule by an interpretation.

\begin{definition}
An interpretation $I$ \textbf{satisfies} a rule
$
r = L \tarrow f, (\gedge,\gnode,h)_\ifl
$
iff for all $v \in V$ and $t \in \mathit{TTS}(I,v,r)$ it holds that
\[
I(t)(v) \models \Big\langle L, \BOUND\big(r,v,I(t-\Delta t)\big)\Big\rangle.
\]
\end{definition}

\begin{example}
\label{satRule}
Let $I_1$ be the interpretation from Example~\ref{satFactEx}.  Suppose that $\<\watchesB, [0.8,1.0]\> \in I(1)(5)$.  In this case, $I_1 \models R_2$.  Let $I_2$ be equivalent to  $I_1$
except that we have
$\<\watchesB, [0.0,0.5]\> \in I_2(1)(3)$.  In this case, $I_2 \not\models R_2$.
\hfill\smallblacksq
\end{example}

We now define satisfaction of programs, and introduce {\em canonical interpretations}, in which time points that are not ``targets'' retain information from the last time step.

\begin{definition} For interpretation $I$ and program $P$:

\smallskip
\noindent
$I$ is a {\bf model} for $P$ iff it satisfies all rules, integrity constraints, and fluent facts in that program, strictly satisfies all non-fluent facts in the program, and for all $L \in \call,$ $c \in \calg$ and $t \notin \mathit{TTS}(I,c,L,P)$, $\<L,[0,1]\> \in I(c)(t)$.

\smallskip
\noindent
$I$ is a \textbf{canonical model} for $P$ iff it satisfies all rules, integrity constraints, and fluent facts in $P$, strictly satisfies all non-fluent facts in $P$, and for all $L \in \call,$ $c \in \calg,$ and $t \notin \mathit{TTS}(I,c,L,P)$, $\<L,[0,1]\> \in I(c)(t)$ when $t=0$ and $\<L,bnd\> \in I(t)(c)$ where $\<L,bnd\> \in I(t-1)(c)$, otherwise.
\end{definition}

\begin{example}
\label{canonEx}
Following from previous examples, if we consider interpretation $I_1$ and program $P=\{F_7,R_2\}$, we have that
$\<\watchesB,[0.0,0.2]\>$ must be in $I_1(1)(2)$ in order for $I_1$ to be canonical.
\hfill\smallblacksq
\end{example}

\subsection{Consistency and Entailment}


In this section we discuss consistency and entailment in $\mancalog$ programs, and explore
the use of minimal models towards computing answers to these problems.

\begin{definition}[Consistency]
A $\mancalog$ program $P$ is (canonically) consistent iff there exists a (canonical) model $I$ of $P$.
\end{definition}

\begin{definition}[Entailment]
A $\mancalog$ program $P$ (canonically) entails $\mancalog$ fact $F$ iff for all (canonical) models $I$ of $P$, it holds that $I \models F$.
\end{definition}

Now we define an ordering over models and define the concept of minimal model.  We then show that if we can find a minimal model then we can answer consistency, entailment, and tight entailment queries. To do so, we first define a pre-order over interpretations.

\begin{definition}[Preorder over Interpretations]
Given interpretations $I,I'$ we say $I \sqsubseteq^{pre} I'$ if and only if for all $t, v, L$ if there exists $\< L, \bnd\> \in I(t)(v)$ then there must exist $\<L, \bnd'\> \in I'(t)(v)$ s.t.\ $\bnd'\subseteq \bnd$.
\end{definition}

Next, we define an equivalence relation for interpretations denoted with $\sim$; we will use the notation $[I]$ for the set of all interpretations equivalent to $I$ w.r.t.~$\sim$.  This allows us to define a partial ordering.

\begin{definition}
Two interpretations $I, I'$ are \textbf{equivalent} (written $I \sim I'$) iff for all $P \in \mathbf{P}$, $I \models P$ iff $I' \models P$.
\end{definition}

\begin{definition}[Partial Ordering]
Given classes of interpretations $[I],[I']$ that are equivalent w.r.t.~$\sim$, we say that $[I]$ precedes $[I']$, written $[I] \sqsubseteq [I']$,
iff $I \sqsubseteq^{pre} I'$.
\end{definition}

The partial ordering is clearly reflexive, antisymmetric, and transitive.  Note that we will often use $I \sqsubseteq I'$ as shorthand for $[I] \sqsubseteq [I']$.
We define two special interpretations, $\bot$ and $\top$, such that $\forall t\in\tau,c\in\calg$, $\bot(t)(c) = \emptyset$ and there exists network atom $\<L,\emptyset\> \in \top(t)(c)$.
Clearly, no other interpretation can be below $\bot$ as the $[\ell,u]$ bound on all network atoms for each time step and each component is $[0,1]$; similarly, no other interpretation is above $\top$, since for any interpretation $I$ for which there exists $\<L,\bnd \> \in I(t)(c)$ where $\bnd \neq \emptyset$, we have $\emptyset\subseteq \bnd$.  We can prove (see the full version of the paper for details) that with $\top$ and $\bot$, $\<\cali,\sqsubseteq \>$ is a complete lattice.  We can now arrive at the notion of \textit{minimal model} for a $\mancalog$ program.

\begin{definition}[Minimal Model]
Given program $P$, the minimal model of $P$ is a (canonical) interpretation $I$ s.t.\ $I \models P$ and for all (canonical) interpretation $I'$ s.t.\ $I' \models P$, 
we have that $I \sqsubseteq I'$.
\end{definition}

Suppose we have some algorithm $A$ that takes as input a program $P$ and returns an interpretation $I$ (where $I$ does not necessarily satisfy $P$) s.t.\ for all $I'$ where $I' \models P$, $I \sqsubseteq I'$.  It is easy to show that if $A(P) \models P$ then $P$ is consistent.  Likewise, if $A(P)= \top$ then $P$ is inconsistent, as all models must then have a tighter weight bound for the network atoms than an invalid interpretation (hence, making such an interpretation invalid as well).  Clearly, any such algorithm $A$ would provide a sound and complete answer to the consistency problem.  Likewise, if we consider the entailment problem, it is easy to show that for fact $F=(\langle L, \bnd\rangle,c):[t_1,t_2]$, $P$ (canonically) entails $F$ iff the minimal model of $P$ (canonically) satisfies $F$.  This is because for minimal model $A(P)$ of $P$, for any time $t \in [t_1,t_2]$, if $A(P)(t)(c) \models \< L, \bnd \>$ then there is network atom $\< L, \bnd' \> \in A(P)(t)(c)$ s.t. $\bnd'\subseteq \bnd$.
We note that for any other interpretation $I$ of $P$ with $\< L, \bnd''\> \in I(t)(c)$ we have that $\bnd' \supseteq \bnd''$.  Hence, having a minimal model allows us to solve any entailment query.  We can think of a minimal model of a $\mancalog$ program as the outcome of a diffusion process in a multi-agent system modeled as a complex network.  Hence, a question such as ``how many agents will adopt the product with a weight of at least $0.9$ in two months?'' can be easily answered once the minimal model is obtained. 

\section{Fixed Point Model Computation}
\label{fpSec}


In this section we introduce a fixed-point operator that produces the non-canonical minimal model of a $\mancalog$ program in polynomial time.  This is followed by an algorithm to find a canonical minimal model also in polynomial time.  First, we introduce three preliminary definitions.

\begin{definition}
For a given $\mancalog$ program $P$, $c \in \calg$, $L \in \call$, and $t\in \tau$ we define function
$\FB(P,c,t,L) =$
\[
\bigcap_{(\<L,\bnd\>,c):[t_1,t_2]\in P\textit{ s.t. }t \in [t_1,t_2]}\bnd
\]
\end{definition}

\begin{definition}
\label{ibound}
For a given $\mancalog$ program $P$, $c \in \calg$, $L \in \call$, and $t\in \tau$ we define function
$\IB(P,c,t,L) =$
\[
\bigcap_{\<L, \bnd\> \hookleftarrow a \in P \textit{ s.t. }I(t)(c)\models a}\bnd
\]
\end{definition}

\begin{definition}
Given $\mancalog$ program $P$, interpretation $I$, $v \in V$, $L \in \call$, and $t\in \tau$, we define
$\RB(P,I,v,t,L) =$
\[
\bigcap_{r\in P\textit{ s.t. }t \in \mathit{TTS}(I,v,L,P)\cap\mathit{TTS}(I,v,r)} \BOUND(r,v,I(t-\Delta t))
\]
\end{definition}

We can now introduce the operator.

\begin{definition}[$\Gamma$ Operator]
For a given $\mancalog$ program $P$, we define the operator $\Gamma_P : \cali \rightarrow \cali$ as follows:
For a given $I$, for each $t \in \tau$, $c\in \calg$, and $L \in \call$, add $\<\call, \bnd\>$ to $\Gamma_P(I)(t)(c)$ where $\bnd$ is defined as:
\begin{eqnarray*}
\bnd &=& \bnd_{prv}\cap \FB(P,c,t,L)\cap\\
&& \IB(P,I,c,t,L)\cap \RB(P,I,c,t,L)
\end{eqnarray*}
where $\<L,\bnd_{prv}\> \in I(t)(c)$.
\end{definition}

It is easy to show that $\Gamma$ can be computed in polynomial time (the proof is in the full version).  Next, we introduce notation for repeated applications of $\Gamma$.

\begin{definition}[Iterated Applications of $\Gamma$]
Given natural number $i > 0$, interpretation $I$, and program $P$, we define $\Gamma_P^i(I)$, the multiple applications
of $\Gamma$, as follows:
\begin{equation*}
\Gamma^i_P(I) = \begin{cases} \Gamma_P(I ) & \text{if $i=1$}  \\ \Gamma_P(\Gamma^{i-1}_P(I)) & \text{otherwise} \end{cases}
\end{equation*}
\end{definition}
We can prove that the iterated $\Gamma$ operator converges after a polynomial number of applications:

\begin{theorem}
\label{gammaPolyConverge}
Given interpretation $I$ and program $P$, there exists a natural number $k$ s.t.\
$\Gamma_P^k(I) = \Gamma_P^{k+1}(I)$, and
\[
k \in O\Big(|P|\cdot\dinm \cdot \tm \cdot |E|\Big)
\]
where $\dinm$ is the maximum in-degree in the network.
\end{theorem}
\begin{proof}[(sketch)]
For a given vertex $i \in V$, we will use the notation $\dini$ to denote the number of incoming neighbors (of any edge type).  First note that for a given $t \in \tau, i \in V,$ and $L \in \call$, a given rule $r$ can tighten the bound on a network atom formed with $L$ no more than $(\dini+1)\cdot (\dinm+1)$ times.  At each application of $\Gamma$, at least one network atom must tighten.  Hence, as there are only  $O\Big(|P|\cdot\dinm \cdot \tm \cdot |E|\Big)$ tightenings possible, this is also the bound on the number of applications of $\Gamma$.
\end{proof}
In the following, we will use the notation $\Gamma^*_P$ to denote the iterated application of $\Gamma$ after a number of steps sufficient for convergence; Theorem~\ref{gammaPolyConverge} means that we can efficiently compute $\Gamma^*_P$.  We also note that as a single application of $\Gamma$ can be computed in polynomial time, this implies that we can find a minimal model of a $\mancalog$ program in polynomial time.  We now prove the correctness of the operator.  We do this first by proving a key lemma that, when combined with a claim showing that for consistent program $P$, $\Gamma^*_P$ is a model of $P$, tells us that $\Gamma^*_P$ is a minimal model for $P$.  Following directly from this, we have that $P$ is inconsistent iff $\Gamma^*_P=\top$.

\begin{lemma}
\label{boundLemma}
If $I \models P$ and $I' \sqsubseteq I$ then $\Gamma(I') \sqsubseteq I$.
\end{lemma}
\begin{theorem}
\label{minModelGamma}
If program $P$ is consistent then $\Gamma^*_P$ is a minimal model for $P$.
\end{theorem}

\noindent
These results, when taken together, prove that tight entailment and consistency problems for $\mancalog$ can be solved in polynomial time, which is precisely what we set out to
accomplish as part of our desiderata described in Section~\ref{sec:criteria}.  Next, we develop an algorithm for the {\em canonical} versions of consistency and tight entailment, and show that we can bound the running time of the algorithm with a polynomial. We also note that subsequent runs of the convergence of $\Gamma$ will likely complete quicker in practice, as the initial interpretation is the last interpretation calculated (cf.\ line~\ref{recalcLine}).  We also show that the interpretation produced by the algorithm is a canonical minimal model.  Following from that, a program is inconsistent iff the algorithm returns $\top$.

\algsetup{indent=1em}
	\begin{algorithm}[t!]
		\caption{ \textsf{CANON\_PROC}}
		\begin{algorithmic}[1]
\small
		\REQUIRE Program $P$
		\ENSURE Interpretation $I$
		\medskip

		\STATE{ $cur\_interp = \Gamma^*_P(\bot)$;}
		\STATE{ Initialize matrix array $cur\_free[\cdot][\cdot]$ where for $v\in V,$ and $L\in \call$, $cur\_free[v][L] = \tau-\mathit{TTS}(cur\_interp,v,L,P)-\{0\}$;}
		\STATE{ Initialize array $vl\_pr[\cdot]$ where for each $t \in [1,\tm]$, $vl\_pr[t]=\{ (v,L) \; | \; t \in cur\_free[v][L] \}$; }
				\FOR{$t=1,\ldots,\tm$}\label{bigFor}
					\IF{$vl\_pr[t] \neq \emptyset$}\label{condLine}
						\FOR{$(v,L) \in vl\_pr[t]$}\label{vlLoop}\label{insideForBegin}
							\STATE{ Remove $\<L,bnd\>$ from $I(t)(v)$; }
							\STATE{ Let $a$ be the atom in $I(t-1)(v)$ of the form $\<L,bnd'\>$; }
							\STATE{ Add $a$ to $I(t)(v)$; }
						\ENDFOR\label{insideForEnd}
						\STATE{Set $cur\_interp = \Gamma^*_P(cur\_interp)$;}\label{recalcLine}
					\ENDIF
					\STATE{For $v\in V$, and $L\in \call$, $cur\_free[v][L] = \tau-\mathit{TTS}(cur\_interp,v,L,P)-\{0,\ldots,t\}$ }\label{newCurFree}
					\STATE{For each $t \in [t+1,\tm]$, $vl\_pr[t]=\{ (v,L) \; | \; t \in cur\_free[v][L] \}$ }
				\ENDFOR
				\RETURN $I$
		\end{algorithmic}
\end{algorithm}

\begin{proposition}
Algorithm \textsf{CANON\_PROC} performs no more than $1+\tm\cdot \min(|\call|,|P|)\cdot |V|$ calculations of the convergence of $\Gamma$.
\end{proposition}

\begin{theorem}
\label{canonSound}
If $P$ is consistent, then $\textsf{CANON\_PROC}(P)$ is the minimal canonical model of $P$.
\end{theorem}

\section{Applications}
\label{sec:learn}

In this section, we will briefly discuss work in progress on how \mancalog\ can be applied
in real world settings.

It is widely acknowledged that modeling influence in multi-agent systems (most usefully modeled
as complex networks) is highly desirable for many practical problems as varied as viral marketing,
prevention of drug use, vaccination, and power plant failure. Though \mancalog\ programs are a
rich model to work with, the acquisition of rules is the principal hurdle to overcome; this is mainly
due to this richness of representation, since for each rule we must provide a set of conditions
on the agents being influenced, conditions on their neighbors and their ties to their neighbors, and
how capable these neighbors are of influencing them. A domain expert is likely able to provide important
insights into these components, but the best way to obtain these rules is undoubtedly to leverage
the presence of large amounts of data in domains like Twitter (with about 340M messages sent per day, available
through public APIs), Facebook (over 950M users with more complex information; not publicly available, but
data can be requested through apps), and blogging and photo hosting sites such as Blogger and Flickr
(which have millions of users as well).

Concretely, we have begun working towards this goal by extracting several time-series, multi-attribute network data sets on
which to apply $\mancalog$.
For instance, to study the proliferation of research on different topics, we looked at research on ``niacin'' indexed by
Thomson Reuters Web of Knowledge (\url{http://wokinfo.com}). This topic was chosen due to its interest to a variety of
disciplines, such as medicine, biology, and chemistry; this gives the data more variety compared with more
discipline-specific topics.
We extracted an author-paper bipartite network consisting of
$3,790$ papers with $10,465$ authors and $16,722$ edges (cf.\ Figure~\ref{fig:papers12}); from this data we can easily
focus on various kinds of networks (co-author, citation, etc.).  We have also collected attribute and time-series data for this network, as well as the subjects of the papers; the propagation of these subjects is a good starting point to test
methods for the acquisition of $\mancalog$ rules. We are harvesting larger datasets from various online social networks.
Further details can be found in the full version of the paper.

\smallskip
\noindent
{\bf A proposed learning architecture.}
We are currently developing a \mancalog\ learning architecture (depicted in Fig.~\ref{fig:architecture})
based on the use of state-of-the-art data analysis, clustering, and influence learning techniques as building blocks for the acquisition of \mancalog\ rules from data sets. The key question is not just the identification of the best
techniques to adopt, but how to adapt them and combine them in such a way as to produce meaningful and useful outputs.

Consider the diagram in Fig.~\ref{fig:architecture}: the data first flows from raw data sources to the {\em cluster identification} component, which has the goal of identifying
sets of agents behaving as groups (for instance, teens influencing other teens of the same sex in the consumption
of music, or scientists of a certain field influencing the research topics of others in a related field)~\cite{warren2005clustering,jain2010data}; the main output here is a set of conditions on nodes and edges that characterize groups of nodes.
Once clusters are identified, the {\em influence recognition} component will make use of both the clusters and
the data sources to recognize what kind of influence is present in the system~\cite{aral12,goyal2010learning,goyal2011data};
the main output of this component is
the influence function to be used in the \mancalog\ rules. The {\em rule generation} component then takes the output
of the cluster identification and influence recognition components, along with the raw data ({\em e.g.}, to analyze
time stamps) and produces \mancalog\ rules; the output of this component is involved in a refinement cycle with
experts who can provide feedback on the rules being produced (such as possible combinations of rules, identification of cases of overfitting, etc.).

\begin{figure}[t!]
  \centering
  \includegraphics[width=0.25\textwidth]{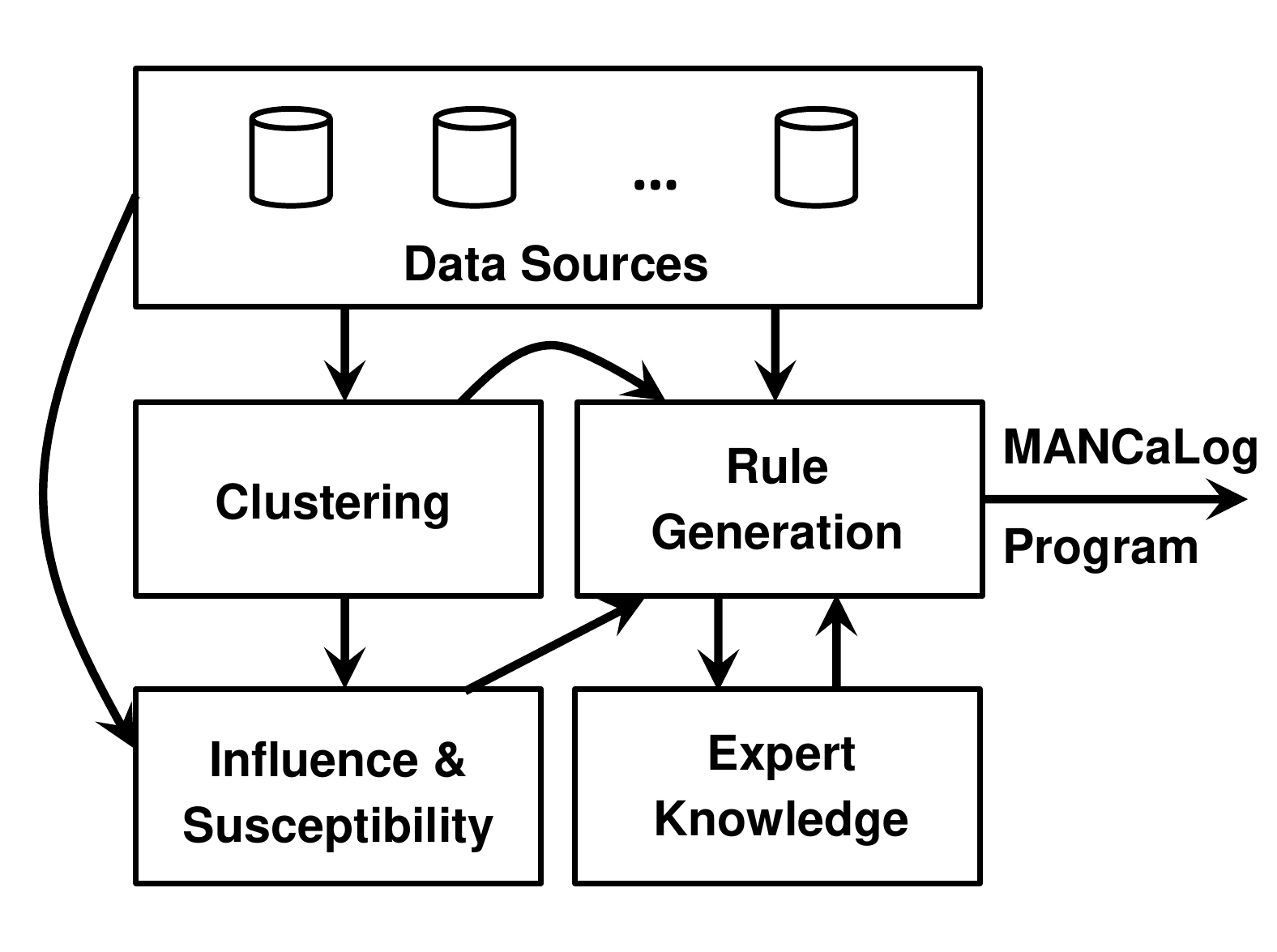}
  \caption{An architecture for obtaining \mancalog\ programs from available data sources.}
  \label{fig:architecture}
\end{figure}


\begin{figure}[t!]
  \centering
  \includegraphics[width=0.45\textwidth]{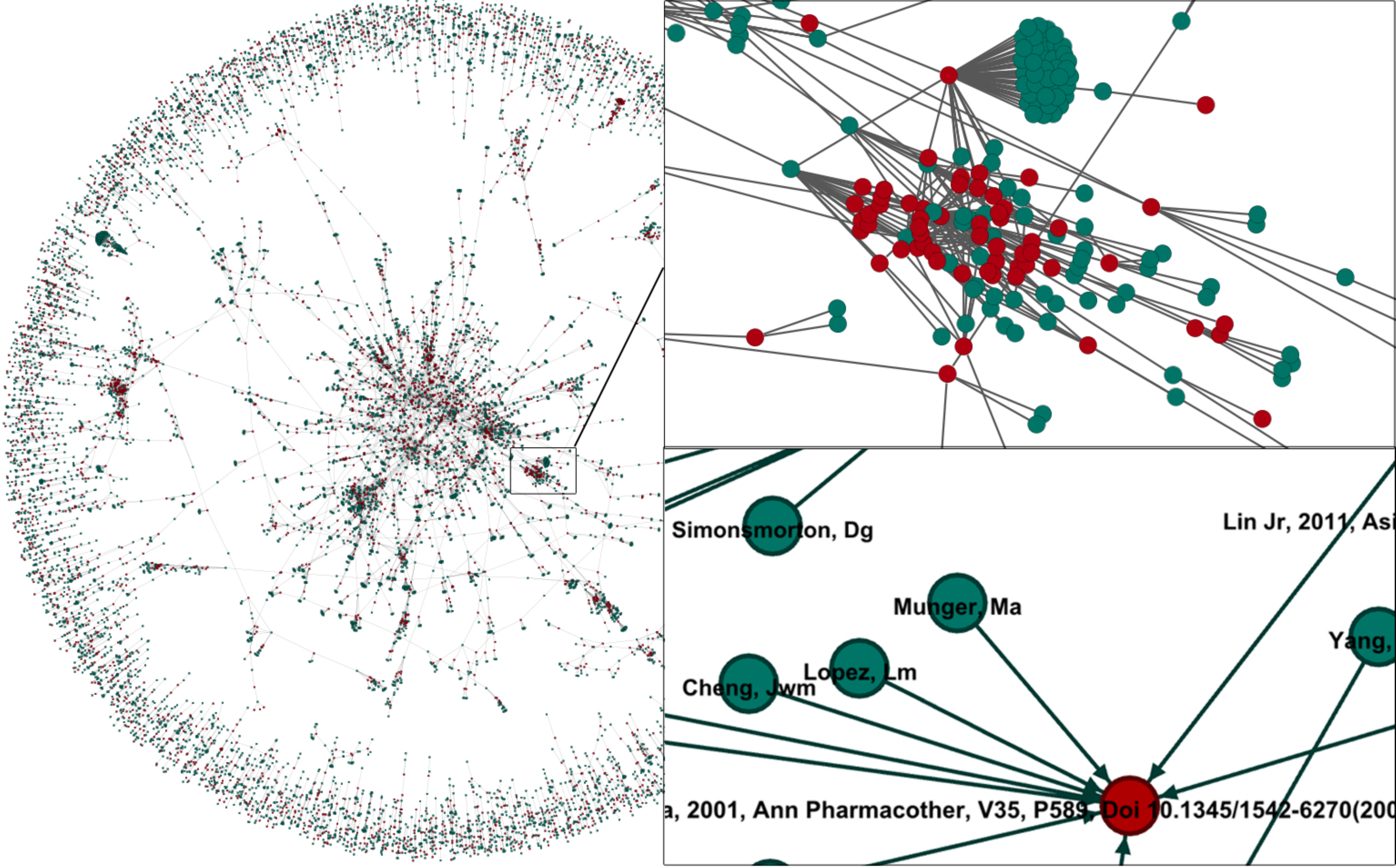}
	\caption{{\em (Left)} Visualization of a multi-attribute time-series author-paper network from 1952 to 2012. 
{\em (Top-Right)} Close-up of the data inside the small box in the main figure. 
{\em (Bottom-Right)} Close-up showing node attributes. In all cases, authors are colored green and 
papers are colored red.
Data extracted from Thomson Reuters Web of Knowledge.}
   \label{fig:papers12}
\end{figure}

\newpage

\section{Conclusion}
\label{sec:conc}

In this paper, we presented the $\mancalog$ language for modeling cascades in multi-agent systems organized in the form of complex networks.  We started by establishing seven criteria in the form of desiderata for such a formalism, and proved that $\mancalog$ meets all of them; to the best of our knowledge, this has not been accomplished by any previous model in the literature.  We also note that $\mancalog$ is the first language of its kind to consider network structure in the semantics, potentially opening the door for algorithms that leverage features of network topology in more efficiently answering queries.  Our current work involves implementing the algorithms described in this paper, as well as the real-world applications described in Section~\ref{sec:learn}; though our algorithms have polynomial time complexity, it is likely that further optimizations will be needed in practice to ensure scalability for very large data sets.

In the near future, we shall also explore various types of queries that have been studied in the literature, such as finding agents of maximum influence, identifying agents that cause a cascade to spread more quickly, and identifying agents that can be influenced in order to halt a cascade. 

\section{Acknowledgments}
P.S. is supported by the Army Research Office (project 2GDATXR042).  G.I.S. is supported under (UK) EPSRC grant EP/J008346/1 -- PrOQAW.  The opinions in this paper are those of the authors and do not necessarily reflect the opinions of the funders, the U.S. Military Academy, or the U.S. Army.

\pagebreak


\newpage

\section{Appendix}
\subsection{Set of interpretations form a complete lattice}

With top interpretation $\top$ and bottom interpretation $\bot$, $\<\cali,\sqsubseteq \>$ is a complete lattice.
\begin{proof}
Let $\cali'$ be a subset of $\cali$.  We can create $\inf(\cali')$ as follows.  We build interpretation $I'$.
For each $t\in \tau, c\in\calg, L \in \call$, let $\ell_1$ be the least of the set
$\cup_{I \in \cali'}\{\ell | \<L, [\ell,u] \> \in I(t)(c),\<L, [\ell,u) \> \in I(t)(c) \}$ and $\ell_2$ be the least of the set
$\cup_{I \in \cali'}\{\ell | \<L, (\ell,u] \> \in I(t)(c),\<L, (\ell,u) \> \in I(t)(c) \}$.
Then, for each $t\in \tau, c\in\calg, L \in \call$ let $u_1$ be the greatest element of the set
$\cup_{I \in \cali'}\{u | \<L, [\ell,u] \> \in I(t)(c),\<L, (\ell,u] \> \in I(t)(c) \}$\\ and $u_2$ be the greatest of the set\\
$\cup_{I \in \cali'}\{u | \<L, [\ell,u) \> \in I(t)(c),\<L, (\ell,u) \> \in I(t)(c) \}$.  If there is any interpretation $I$ in $\cali$ where there is not some $\bnd$ s.t. $\<L, \bnd \>\in I(t)(c)$ then add $\<L, [0,1] \>$ to $I'(t)(c)$.  If $\ell_2\leq \ell_1$ and $u_1 \geq u_2$ then add $\<L, (\ell_2,u_1] \>$ to $I'(t)(c)$.  If $\ell_2\leq \ell_1$ and $u_2 > u_1$ then add $\<L, (\ell_2,u_2) \>$ to $I'(t)(c)$.  If $\ell_2> \ell_1$ and $u_2 > u_1$ then add $\<L, [\ell_1,u_2) \>$ to $I'(t)(c)$.  Finally, if $\ell_2> \ell_1$ and $u_1 \geq u_2$ then add $\<L, [\ell_1,u_1] \>$ to $I'(t)(c)$.  Clearly, $I' = \inf(\cali')$.

In the next part of the proof, we show we can create $\sup(\cali')$ as follows.  We build interpretation $I'$.  For each $t\in \tau, c\in\calg, L \in \call$ let $\ell_1$ be the greatest of the set\\ $\cup_{I \in \cali'}\{\ell | \<L, [\ell,u] \> \in I(t)(c),\<L, [\ell,u) \> \in I(t)(c) \}$ and $\ell_2$ be the greatest of the set\\ $\cup_{I \in \cali'}\{\ell | \<L, (\ell,u] \> \in I(t)(c),\<L, (\ell,u) \> \in I(t)(c) \}$.  Then, for each $t\in \tau, c\in\calg, L \in \call$ let $u_1$ be the least element of the set $\cup_{I \in \cali'}\{u | \<L, [\ell,u] \> \in I(t)(c),\<L, (\ell,u] \> \in I(t)(c) \}$ and $u_2$ be the least of the set $\cup_{I \in \cali'}\{u | \<L, [\ell,u) \> \in I(t)(c),\<L, (\ell,u) \> \in I(t)(c) \}$.  If $\max(\ell_1,\ell_2)>\min(u_1,u_2)$ or $(\ell_2>\ell_1) \wedge (u_2<u_1)\wedge(\ell_2=u_2)$ then add $\<L, \emptyset \>$ to $I'(t)(c)$.  If $\ell_2> \ell_1$ and $u_1 \leq u_2$ then add $\<L, (\ell_2,u_1] \>$ to $I'(t)(c)$.  If $\ell_2> \ell_1$ and $u_2 < u_1$ then add $\<L, (\ell_2,u_2) \>$ to $I'(t)(c)$.  If $\ell_2\leq \ell_1$ and $u_2 < u_1$ then add $\<L, [\ell_1,u_2) \>$ to $I'(t)(c)$.  Finally, if $\ell_2\leq \ell_1$ and $u_1 \leq u_2$ then add $\<L, [\ell_1,u_1] \>$ to $I'(t)(c)$.  Clearly, $I' = \sup(\cali')$.\\
\indent As both $\inf(\cali')$ and $\sup(\cali')$ exist and are clearly in $\cali$ then the statement follows.
\end{proof}


\subsection{A single application of $\Gamma$ can be computed in polynomial time}
For interpretation $I$, $\Gamma(I)$ can be computed by conducting $O(|P|\cdot |V| \cdot \tm \cdot \dinm)$ satisfaction checks where $\dinm$ is the maximum in-degree of a node in the network.  (This combined with the assumption that the influence function is computed in constant time results in polynomial time computation for a single application of $\Gamma$.)

\begin{proof}
We note that a given rule will require the most satisfaction checks, as a rule will potentially affect a network atom of a certain label for each vertex-time point pair.  By the definition of $\RB$, a given rule clearly causes no more than $O(\dinm)$ satisfaction checks.  As the number of rules is no more than $|P|$, the statement follows.
\end{proof}

\pagebreak

\subsection{Proof of Theorem~\ref{gammaPolyConverge}}

Given interpretation $I$ and program $P$, there exists a natural number $k$ s.t.\
$\Gamma_P^k(I) = \Gamma_P^{k+1}(I)$, and
\[
k \in O\Big(|P|\cdot\dinm \cdot \tm \cdot |E|\Big)
\]
where $\dinm$ is the maximum in-degree in the network.

\begin{proof}
For a given vertex $i \in V$, we will use the notation $\dini$ to denote the number of incoming neighbors (of any edge type) and $\dinm = \max_i \dini$.  First we show that for a given $t \in \tau, i \in V,$ and $L \in \call$, a given rule $r$ can tighten the bound on a network atom formed with $L$ no more than $(\dini+1)\cdot (\dinm+1)$ times.  This is because a given rule adjusts the bound on a network atom based on the number of eligible and qualifying neighbors, which are bounded by $\dini,\dinm$ respectively.  At each application of $\Gamma$, at least one network atom must tighten.  Hence, as there are only  $O\Big(|P|\cdot\dinm \cdot \tm \cdot \sum_i \dini \Big)=$$O\Big(|P|\cdot\dinm \cdot \tm \cdot |E|\Big)$ tightenings possible, this is also the bound on the number of applications of $\Gamma$.
\end{proof}

\subsection{Proof of Lemma~\ref{boundLemma}}

If $I \models P$ and $I' \sqsubseteq I$ then $\Gamma(I') \sqsubseteq I$.
\begin{proof}
Suppose, BWOC, that $\Gamma(I') \sqsupset I$.  Then, there exists some $L \in \call$, $t\in \tau$ and $c \in \calg$ s.t.\ $\<L, \bnd \> \in I(t)(c)$,  $\<L, \bnd' \> \in I'(t)(c)$, and  $\<L, \bnd'' \> \in \Gamma(I')(t)(c)$ s.t.\ $\bnd \supset \bnd''$ and $\bnd' \supseteq \bnd''$.  There are four things that affect $\bnd''$: facts, rules, integrity constraints and $\bnd'$.  Clearly, we need not consider the effect that either facts or $\bnd'$ have on $\bnd''$, as $I$ satisfies all facts and $I' \sqsubseteq I$.  We also note that a given integrity constraint imposed by Definition~\ref{ibound} can tighten $\bnd''$ no more than the associated bound in any model.  Hence, there must be some rule $r=L \tarrow f, (\gedge,\gnode,h)_\ifl$ that causes $\bnd''$ to become less than $\bnd$.  As $\bnd'' \neq \bnd'$, we know that $t \in \mathit{TTS}(\Gamma(I'),c,r)\cap \mathit{TTS}(I',c,r)$.  As a result, we have $\Gamma(I')(t-\Delta t)(c) \models f$ and $I'(t-\Delta t)(c) \models f$.  Further, as $I \models P, I' \sqsubseteq I,$ and no rule can modify a non-fluent atom, we have
\begin{quote}
$|\ELIG(v,\gedge,\gnode,\Gamma(I')(t-\Delta t)| =$
\end{quote}
\begin{quote}
$|\ELIG(v,\gedge,\gnode,I'(t-\Delta t)| =$
\end{quote}
\begin{quote}
$|\ELIG(v,\gedge,\gnode,\Gamma(I')(t-\Delta t)|.$
\end{quote}
Further, we know that as $I' \sqsubseteq I$, it must be the case that
\[
|\QUAL(v,\gedge,\gnode,h,I(t-\Delta t))| \geq
\]
\[
|\QUAL(v,\gedge,\gnode,h,I'(t-\Delta t))|.
\]
This implies, by Axiom~2 that, $\BOUND(r,v,I(t-\Delta t)) \subseteq \BOUND(r,v,I'(t-\Delta t))$.  This then implies that $\bnd \subseteq \bnd''$, which is a contradiction.
\end{proof}


\subsection{Proof of Theorem~\ref{minModelGamma}}

$\Gamma^*_P$ is a minimal model for $P$.
\begin{proof}
\noindent Claim: If program $P$ is consistent then $\Gamma^*_P$ is a model of $P$.\\
Suppose, BWOC, that there is a fact in $P$ that $\Gamma^*_P$ does not satisfy.  However, by the definition of $\Gamma$ and the definition of a fact, $\Gamma^*_P$ must satisfy all facts as the bound on the weight associated with each fact is included in the intersection.  Further, we can also see by the definition of $\Gamma$ that $\Gamma^*_P$ strictly satisfies all non-fluent facts in $P$.  We also note that the final application of the $\Gamma$ operator ensures that all integrity constraints are satisfied by $\Gamma^*_P$.  Now, Suppose, BWOC, that there is a rule in $P$ that $\Gamma^*_P$ does not satisfy.  However, with each application of $\Gamma$, for each rule, we include the bound on the weight returned by the $\BOUND$ function for each time step in the target time step associated with that rule.  As $\Gamma$ is applied to convergence, and new bounds are intersected with each application, then we know that all time points in any associated target time set are considered in the intersection.\\
\noindent Proof of Theorem: The above claim tells us that $\Gamma^*_P \models P$.  Now consider interpretation $I$ s.t.\ $I \models P$.  As $\bot \sqsubseteq I$, multiple applications of Lemma~\ref{boundLemma} tell us that $\Gamma^*_P \sqsubseteq I$.  Hence, the statement follows.
\end{proof}

\subsection{Proof of Theorem~\ref{canonSound}}

If $P$ is consistent, then $\textsf{CANON\_PROC}(P)$ is the minimal canonical model of $P$.

\begin{proof}
\noindent CLAIM 1: If $P$ is consistent, then $\textsf{CANON\_PROC}(P)$ is a canonical model of $P$.\\
Clearly, $I=\textsf{CANON\_PROC}(P)$ satisfies all facts and integrity constraints in $P$.  Hence, we shall consider programs that only consist of rules in this proof.  We say $I$ $L$-canonically satisfies $P$ iff $I$ canonically satisfies $\{ r \in P \; | \; \textit{head}(r)=L\}$.  Clearly, $I$ canonically satisfies $P$ if for all $L\in\call$, $P$ $L$-canonically satisfies  by $I$. We say that $I$ is an $(L,c,q)$-canonically consistent interpretation if for $c \in \calg$, for the first $t \in \tau - \mathit{TTS}(I,c,L,P)-\{0\}$, $I(t)(c)\models \<L,bnd\>$ where $\<L,bnd\>\in I(t-1)(c)$.  Consider some $L\in \call$ and $c \in \calg$.  Clearly, $I$ is an $(L,c,0)$-model for $P$.  Let us assume, for some value $q$, that $I$ is an $(L,c,q-1)$ model for $P$.  Let time point $t$ be the $q$-th element of $\tau - \mathit{TTS}(I,c,L,P)-\{0\}$.  Consider the time step before time $t$ is considered in the for-loop at line~\ref{bigFor} of \textsf{CANON\_PROC}, which causes the condition at line~\ref{condLine} to be true.  By line~\ref{newCurFree}, $\tau - \mathit{TTS}(I,c,L,P)-\{0\} \subseteq cur\_free[c][L]$.  This means that $t$ is a member of both.  Hence, when $t$ is considered at line~\ref{bigFor}, the condition at line~\ref{condLine} is true, causing $\<L,bnd\> \in I(t)(c)\cap I(t-1)(c)$ and as the element $\<L,bnd\> \in I(t-1)(c)$ is not changed here, we have shown the $I$ is an $(L,c,q)$-model for $P$.  By the for-loop at line~\ref{vlLoop}, for all $L\in\call$ and $c\in\calg$, $I$ is an $(L,c,q)$-model for $P$.  Hence, at the for-loop at line~\ref{bigFor}, we can be assured that for $L\in\call$ and $c\in\calg$ that $I$ $(L,c,|\tau - \mathit{TTS}(I,c,L,P)-\{0\}|)$ satisfies $P$ -- which means that $I$ canonically satisfies $P$\\
\noindent CLAIM 2: If $I$ is a canonical model for $P$,\\ $cur\_interp \sqsubseteq I$ is an interpretation that also strictly satisfies all non-fluent facts in $P$, and $cur\_interp'$ is $cur\_interp$ after being manipulated in lines~\ref{insideForBegin}-\ref{insideForEnd} of \textsf{CANON\_PROC}, then $cur\_interp' \sqsubseteq I$.

We note that by the definition of satisfaction of a non-fluent fact, and the fact that both $cur\_interp$ and $I$ must strictly satisfy all non-fluent facts in $P$, we know that for all $c \in \calg$ and $L \in \call$ that:
\begin{eqnarray*}
TTS(I,c,L,P) &=&TTS(cur\_interp,c,L,P)\\
 &=& TTS(cur\_interp',c,L,P)
 \end{eqnarray*}
Let us assume that lines~\ref{insideForBegin}-\ref{insideForEnd} of the algorithm are changing $cur\_interp$ when the outer loop is considering time $t$ and that the condition at line~\ref{condLine} is true.  Clearly,
\begin{eqnarray*}
t \in \tau-TTS(I,v',L',P)-\{0\}
\end{eqnarray*}
As a result, for any $(v,L)$ pair considered at this point by the algorithm, if $\<L, \bnd \> \in I(t)(v)$ and $\<L, \bnd' \> \in I(t-1)(v)$ then we have $\bnd = \bnd'$.  By the algorithm, if we have $\<L, \bnd^{*} \> \in cur\_interp'(t)(v)$ and $\<L, \bnd^{**} \> \in cur\_interp'(t-1)(v)$ we have
that $\bnd^{*}=\bnd^{**}$.  As\\ $\<L, \bnd^{**} \> \in cur\_interp(t-1)(v)$, we know that $\bnd' \subseteq \bnd^{**}$.  As a result, we have $cur\_interp' \sqsubseteq I$, completing the claim.

\noindent Proof of theorem:  As initially $cur\_interp = \Gamma^*_P$ and $\Gamma^*_P \sqsubseteq I$ by Theorem~\ref{minModelGamma}, we note that the algorithm changes\\ $cur\_interp$ either by applying $\Gamma$ or manipulating it in lines~\ref{insideForBegin}-\ref{insideForEnd}, which tells us (by claim 2) that for all models $I$ of $P$ that\\ $\textsf{CANON\_PROC}(P) \sqsubseteq I$.  Since by claim 1 we know that\\ $\textsf{CANON\_PROC}(P) \models P$, the statement of the theorem follows.
\end{proof}

\subsection{Details on the Extracted Dataset}

One way in which \mancalog\ can be used is looking at proliferation of research on different topics. We look at research conducted on niacin, an organic compound commonly used for increasing levels of high density lipoproteins (HDL). Using Thomson Reuters Web of Knowledge (\url{http://wokinfo.com}) we were able to extract information on $4,202$ articles about niacin. This information was then processed using the Science of Science (Sci$^2$) Tool (\url{http://sci2.cns.iu.edu})  to extract numerous different networks such as author by paper networks, citation networks, and paper by subject networks. Each paper has attributes about when it was published, what journal it was published in, and what subjects the paper was about. During the first time period there is a total of $508$ papers with $856$ different authors and $1,231$ connections based on an author being cited as an author of a given paper. During the second time period, there is a total of $3,790$ papers with $10,465$ different authors and $16,772$ connections.

%

\end{document}